\documentclass[11pt,a4paper]{article}
\usepackage[hyphens]{url}
\usepackage[hyperref]{emnlp2020}
\usepackage{times}
\usepackage{latexsym}

\usepackage{xspace}
\usepackage{graphicx}
\usepackage{subcaption}
\usepackage[utf8]{inputenc}
\usepackage[export]{adjustbox}
\usepackage{wrapfig}
\usepackage{tabularx}
\usepackage{array}
\usepackage{amssymb}
\usepackage{amsmath}
\usepackage{dsfont}
\usepackage{booktabs}
\newenvironment{remark}[1][Remark]{\begin{trivlist}
\item[\hskip \labelsep {\bfseries #1}]}{\end{trivlist}}
\usepackage{bbm}
\usepackage{amsthm}
\newtheorem{theorem}{Theorem}[section]
\newtheorem{lemma}{Lemma}[section]
\newtheorem{assumption}{Assumption}[section]
\usepackage{color}
\usepackage{algpseudocode}
\usepackage{algorithm}
\usepackage{enumitem} 

\DeclareMathOperator*{\argmin}{arg\,min}
\DeclareMathOperator*{\argmax}{arg\,max}

\newcommand{\mispl}{MISP-\textsc{Neil}\xspace}
\newcommand{\mispls}{MISP-\textsc{Neil}*\xspace}

\newcommand{\neil}{\textsc{Neil}\xspace}

\newcommand{\nop}[1]{}

\usepackage{microtype}

\aclfinalcopy 

\title{An Imitation Game for Learning Semantic Parsers from User Interaction}

\author{
    Ziyu Yao\textsuperscript{1}, Yiqi Tang\textsuperscript{1}, Wen-tau Yih\textsuperscript{2}, Huan Sun\textsuperscript{1}, Yu Su\textsuperscript{1}\\
    {\tt \{yao.470, tang.1466, sun.397, su.809\}@osu.edu}\\
    {\tt scottyih@fb.com}\\
    \textsuperscript{1}The Ohio State University\\
    \textsuperscript{2}Facebook AI Research, Seattle\\
}

\date{}

\begin{document}
\maketitle
\begin{abstract}

Despite the widely successful applications, building a semantic parser is still a tedious process in practice with challenges from costly data annotation and privacy risks. We suggest an alternative, human-in-the-loop methodology for learning semantic parsers directly from users. A semantic parser should be introspective of its uncertainties and prompt for user demonstrations when uncertain. In doing so it also gets to imitate the user behavior and continue improving itself autonomously with the hope that eventually it may become as good as the user in interpreting their questions. To combat the sparsity of demonstrations, we propose a novel \emph{annotation-efficient imitation learning} algorithm, which iteratively collects new datasets by mixing demonstrated states and confident predictions and retrains the semantic parser in a Dataset Aggregation fashion~\cite{ross2011reduction}. We provide a theoretical analysis of its cost bound and also empirically demonstrate its promising performance on the text-to-SQL problem.\footnote{Code will be available at \url{https://github.com/sunlab-osu/MISP}.}
\end{abstract}

\section{Introduction}\label{sec:intro}
\begin{figure}[t!]
    \centering
    \includegraphics[width=\linewidth]{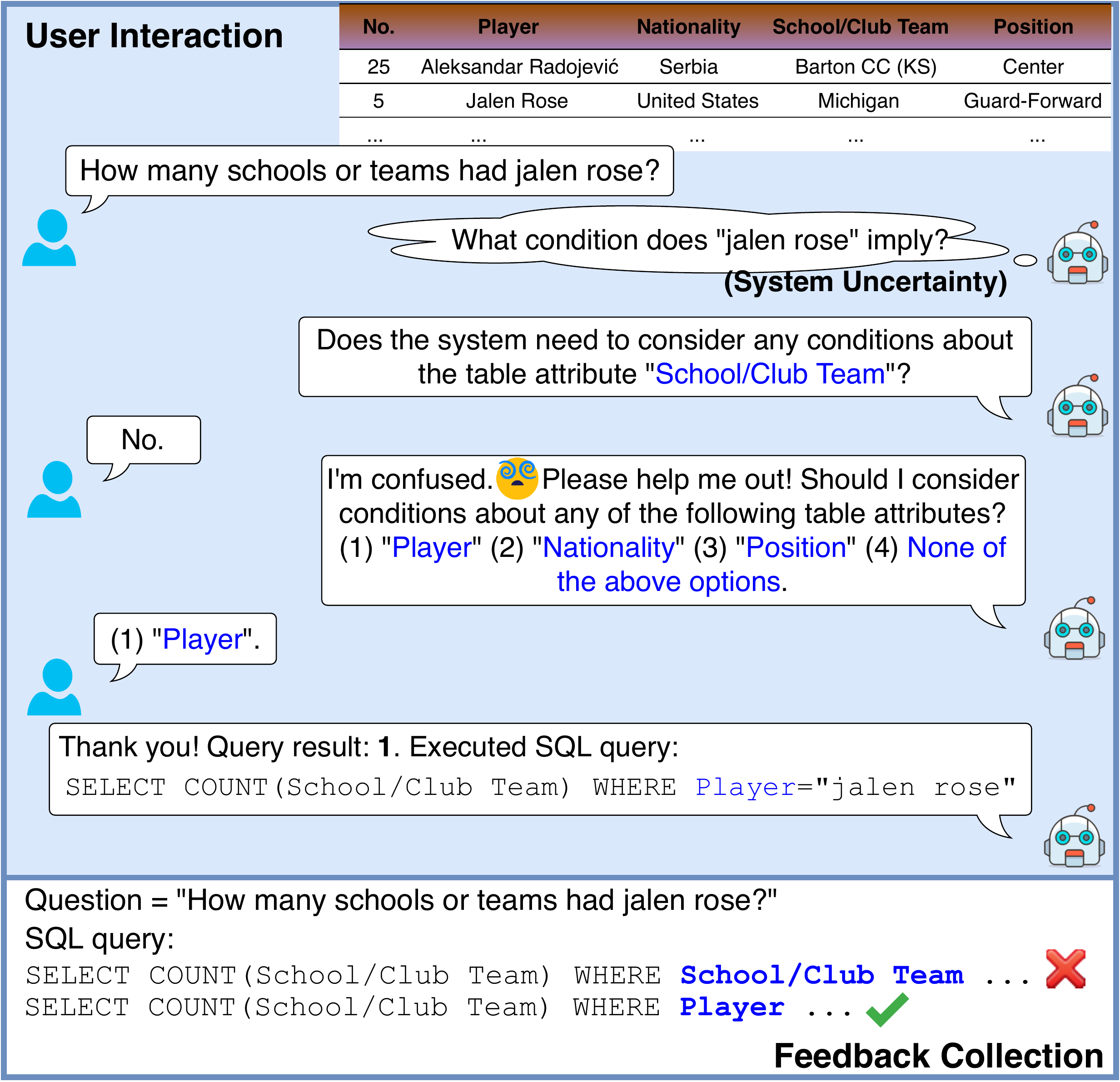}
    \caption{A semantic parser proactively interacts with the user in a friendly way to resolve its uncertainties.\nop{ In doing so it also accumulates targeted training data and continues improving itself autonomously without annotators or developers.} In doing so it also gets to imitate the user behavior and continue improving itself autonomously with the hope that eventually it may become as good as the user in interpreting their questions.}
    \label{fig:intro}
    \vspace{-10pt}
\end{figure}

Semantic parsing has found tremendous applications in building natural language interfaces that allow users to query data and invoke services without programming~\cite{woods1973progress,zettlemoyer2005learning,berant2013semantic, yih2015semantic, su2017building,yu2018spider}. The life cycle of a semantic parser typically consists of two stages: (1) \emph{bootstrapping}, where we keep collecting labeled data via trained annotators and/or crowdsourcing for model training until it reaches commercial-grade performance (e.g., 95\% accuracy on a surrogate test set), and (2) \emph{fine-tuning}, where we deploy the system, analyze the usage, and collect and annotate new data to address the identified problems or emerging needs.
However, it poses several challenges for scaling up or building semantic parsers for new domains: 
(1) \emph{high bootstrapping cost} because mainstream neural parsing models are data-hungry and the annotation cost of semantic parsing data is relatively high, 
(2) \emph{high fine-tuning cost} from continuously analyzing usage and annotating new data, and 
(3) \emph{privacy risks} arising from exposing private user data to annotators and developers~\cite{techcrunch2019privacy}. 

In this paper, we suggest an alternative methodology for building semantic parsers that could potentially address all the aforementioned problems. The key is to involve human users in the learning loop.
A semantic parser should be introspective of its uncertainties~\cite{dong2018confidence} and proactively prompt for demonstrations from the user, who knows the question best, to resolve them. In doing so, the semantic parser \nop{would be able to}can accumulate targeted training data and continue improving itself autonomously without involving any annotators or developers, hence also minimizing privacy risks. The bootstrapping cost could also be significantly reduced because an interactive system needs not to be almost perfectly accurate to be deployed. On the other hand, such interaction opens up the black box and allows users to know more about the reasoning underneath the system and better interpret the final results~\cite{su2018natural}. A human-in-the-loop methodology like this also opens the door for domain adaptation and personalization. 

This work builds on the recent line of research on interactive semantic parsing~\cite{li2014constructing, chaurasia2017dialog, gur2018dialsql, yao2019model}. Specifically, \citet{yao2019model} provide a general framework, MISP (Model-based Interactive Semantic Parsing), which handles uncertainty modeling and natural language generation. We will leverage MISP for user interaction to prove the feasibility of the envisioned methodology. However, existing studies only focus on interacting with users to resolve uncertainties. None of them has fully addressed the crucial problem of \emph{how to continually learn from user interaction}, which is the technical focus of this study.
\nop{None of them has answered the crucial question of \emph{how to continually learn from user interaction}, which is the technical focus of this study.} 

One form of user interaction explored for learning semantic parsers is asking users to validate the execution results~\cite{clarke2010driving,artzi2013weakly,iyer2017learning}. While appealing, in practice it may be a difficult task for real users because they would not need to ask the question if they knew the answer in the first place. We instead aim to learn semantic parsers from fine-grained interaction where users only need to answer simple questions covered by their background knowledge (Figure~\ref{fig:intro}). However, learning signals from such fine-grained interactions are bound to be sparse because the system needs to avoid asking too many questions and overwhelming the user, which poses a challenge for learning. 

To tackle the problem, we propose \neil, a novel \emph{a\underline{N}notation-\underline{E}fficient \underline{I}mitation \underline{L}earning} algorithm for learning semantic parsers from such sparse, fine-grained demonstrations: The agent (semantic parser) only requests for demonstrations when it is uncertain about a state (parsing step).
For certain/confident states, actions chosen by the current policy are deemed correct and are executed to continue parsing. The policy is updated iteratively in a Dataset Aggregation fashion~\cite{ross2011reduction}. In each iteration, all the state-action pairs, demonstrated or confident, are included to form a new training set and train a new policy in a supervised way. Intuitively, using confident state-action pairs for training mitigates the sparsity issue, but it may also introduce training bias. We provide a theoretical analysis and show that, under mild assumptions, the impact of the bias and the quality of the \neil policy can be controlled by tuning the policy initialization and confidence estimation accuracy.

We also empirically compare \neil with a number of baselines on the text-to-SQL parsing task. On the WikiSQL~\cite{zhong2017seq2sql} dataset, we show that, \emph{when bootstrapped using only 10\% of the training data, \neil can achieve almost the same test accuracy (2\% absolute loss) as the full expert annotation baseline, while requiring less than 10\% of the annotations that the latter needs}, without even taking into account the different unit cost of annotations from users vs.\ domain experts. We also show that the quality of the final policy is largely determined by the quality of the initial policy, which provides empirical support for the theoretical analysis. Finally, we demonstrate that \neil can generalize to more complex semantic parsing tasks such as Spider~\cite{yu2018spider}.

\section{Related Work}\label{sec:related_work}
\begin{remark}[Interactive Semantic Parsing.] 
Our work extends the recent idea of leveraging system-user interaction to improve semantic parsing on the fly~\cite{li2014constructing,he2016human,chaurasia2017dialog,su2018natural, gur2018dialsql,yao2018interactive,yao2019model,elgohary2020speak, zeng2020photon, SMDataflow2020}. \citet{gur2018dialsql} built a neural model to identify and correct error spans in generated queries via dialogues. \citet{yao2019model} formalized a model-based intelligent agent MISP, which enables user interaction via a policy probability-based uncertainty estimator, a grammar-based natural language generator, and a multi-choice question-answer interaction design. More recently, \citet{elgohary2020speak} crowdsourced a dataset for fixing incorrect SQL queries using free-form natural language feedback. {\citet{SMDataflow2020} constructed a contextual semantic parsing dataset where agents could trigger conversations to handle exceptions such as ambiguous or incomplete user commands.}
In this work, we seek to continually improve semantic parsers from such user interaction, a topic that is not carefully studied by the aforementioned work.
\end{remark}

\begin{remark}[Interactive Learning from Feedback.] 
Learning interactively from user feedback has been studied in many NLP tasks \cite{sokolov2016learning, wang2016learning, wang2017naturalizing, nguyen2017reinforcement, gao2018april, abujabal2018never, hancock2019learning, kreutzer2019self}. Most relevant to us, \citet{hancock2019learning} constructed a self-feeding chatbot that improves itself from user satisfied responses and their feedback on unsatisfied ones.

In the field of semantic parsing, \citet{clarke2010driving, artzi2013weakly, iyer2017learning} learned semantic parsers from binary user feedback on whether executing the generated query yields correct results. However, often times {(especially in information-seeking scenarios)} it may not be very practical to expect end users able to validate the denotation correctness (e.g., consider validating an execution result ``\emph{103}'' for the question ``\emph{how many students have a GPA higher than 3.5}'' from a massive table). Active learning is also leveraged to save human annotations \cite{duong2018active,ni20aaai}.
Our work is complementary to this line of research as we focus on learning interactively from end users (not ``teachers'').
\end{remark}

\begin{remark}[Imitation Learning.]
Traditional imitation learning algorithms \cite{daume2009search,ross2010efficient,ross2011reduction,Ross2014ReinforcementAI} iteratively execute and train a policy by collecting expert demonstrations for every policy decision. Despite its efficacy, the learning demands costly annotations from experts. In contrast, we save expert effort by selectively requesting demonstrations. This idea is related to active imitation learning~\cite{chernova2009interactive,kimmaximum,judah2014active,zhang2017query}. For example, \citet{judah2014active} assumed a ``teacher'' and actively requested demonstrations for most informative trajectories in the unlabeled data pool.
Similar to us, \citet{chernova2009interactive} solicited demonstrations only for uncertain states. However, their algorithm simply abandons policy actions that are confident, leading to sparse training data. Instead, our algorithm utilizes confident policy actions to combat the sparsity issue and is additionally provided with a theoretical analysis.

{Concurrent with our work, \citet{daume20active} studied active imitation learning for structured prediction tasks such as named entity recognition. Our work instead focuses on semantic parsing, which presents a unique challenge of \emph{integrality}, i.e., the output sequence (a semantic parse) could only be correct \emph{as a whole} (as opposed to \emph{partially} correct) in order to yield the correct denotation. We therefore propose a new cost function (Section~\ref{sec:theoretical_analysis}) to theoretically analyze the factors that affect the efficacy of learning semantic parsers via imitation.}

\end{remark}

\section{Preliminaries} \label{sec:overview}

Formally, we assume the semantic parsing model generates a semantic parse by executing a sequence of actions $a_t$ (parsing decisions) at each time step $t$. In practice, the definition of an action depends on the specific semantic parsing model, as we will illustrate shortly. A state $s_t$ is then defined as a tuple of $(q, a_{1:t-1})$, where $q$ is the initial natural language question and $a_{1:t-1}=(a_1, ...,a_{t-1})$ is the current partial parse. In particular, the initial state $s_1=(q, \phi)$ contains only the question.
Denote a semantic parser as $\hat{\pi}$, which is a \emph{policy function} \cite{sutton2018reinforcement} that takes a state $s_t$ as input and outputs a probability distribution over the action space. The semantic parsing process can be formulated as sampling a \emph{trajectory} $\tau$ by alternately observing a state and sampling an action from the policy, i.e., $\tau=(s_1$, $a_1 \sim \hat{\pi}(s_1)$, ..., $s_T$, $a_T \sim \hat{\pi}(s_T))$, assuming a trajectory length $T$. The probability of the generated semantic parse becomes: $p_{\hat{\pi}}(a_{1:T}|s_1) = \prod_{t=1}^T p_{\hat{\pi}}(a_t|s_t)$.

An interactive semantic parser typically follows the aforementioned definition and requests the user's validation of a specific action $a_t$. Based on the feedback, a correct action $a_t^*$ can be inferred to replace the original $a_t$. The parsing process continues with $a_t^*$ afterwards.

{In this work, we adopt MISP \cite{yao2019model} as the back-end interactive semantic parsing framework, given that it is a principled framework for this purpose and can generalize to various kinds of semantic parsers and logical forms. However, we note that our proposed algorithm is not limited to MISP; it instead depicts a general algorithm for learning semantic parsers from user interaction. We illustrate the application of MISP to {a sketch-based parser,} SQLova \cite{hwang2019comprehensive}, as follows. More details and another example of how it applies to a non-sketch-based parser EditSQL \cite{zhang2019editing} can be found in Appendix~\ref{app:misp}.
}

\begin{remark}[Example.] 
Consider the SQLova parser, which generates a query by filling ``slots'' in a pre-defined SQL sketch ``\texttt{SELECT Agg SCol WHERE WCol OP VAL}''.\nop{In practice, SQLova takes an additional step to predict the number of \texttt{WHERE} conditions. Since it is vague to request user feedback on such a prediction, we consider it as an ``implicit'' decision that controls the explicit SQL generation steps.} To complete the SQL query in Figure~\ref{fig:intro}, it first takes three steps: \texttt{SCol}=``\texttt{School/Club Team}'' ($a_1$), \texttt{Agg}=``\texttt{COUNT}'' ($a_2$) and \texttt{WCol}=``\texttt{School/Club Team}'' ($a_3$). MISP detects that $a_3$ is uncertain because its probability is lower than a pre-specified threshold. It validates $a_3$ with the user and corrects it with \texttt{WCol}=``\texttt{Player}'' ($a_3^*$). The parsing continues with
\texttt{OP}=``\texttt{=}'' ($a_4$) and \texttt{VAL}=``\texttt{jalen rose}'' ($a_5$). Here, the trajectory length $T=5$.
\end{remark}

\section{Learning Semantic Parsers from User Interaction} \label{sec:model}

In this section, we present \neil, an \textit{a\underline{N}notation-\underline{E}fficient \underline{I}mitation \underline{L}earning} algorithm that trains a parser from user interaction, without requiring a large amount of user feedback (or ``annotations''). This property is particularly important for end user-facing systems in practical use.
Note that while we apply \neil to semantic parsing in this work, in principle it can also be applied to other structured prediction tasks (e.g., machine translation).

\subsection{An Imitation Learning Formulation}
Under the interactive semantic parsing framework, a learning algorithm intuitively can aggregate $(s_t, a_t^*)$ pairs collected from user interactions and trains the parser to enforce $a_t^*$ under the state $s_t=(q, a_{1:t-1})$. However, this is not achievable by conventional supervised learning since the training needs to be conducted in an \emph{interactive} environment, where the partial parse $a_{1:t-1}$ is generated by the parser itself.

Instead, we formulate it as an imitation learning problem \cite{daume2009search, ross2010efficient}.
Consider the user as a \emph{demonstrator}, then the derived action $a_t^*$ can be viewed as an \emph{expert demonstration} which is interactively sampled from the demonstrator's policy (or \emph{expert policy}) $\pi^*$,\footnote{We follow the imitation learning literature and use ``expert'' to refer to the imitation target, but the user in our setting by no means needs to be a ``domain (SQL) expert''.} i.e., $a_t^* \sim \pi^*(s_t)$. The goal of our algorithm is thus to train policy $\hat{\pi}$ to imitate the expert policy $\pi^*$. A general procedure is described in Algorithm~\ref{alg:mispl_general} (Line 1--9), where $\hat{\pi}$ is learned iteratively for every $m$ user questions. In each iteration, the policy is retrained on an aggregated training data over the past iterations, following the Dataset Aggregation fashion in \cite{ross2011reduction}.

\begin{algorithm}[t!]
  \begin{algorithmic}[1]
  \Require Initial training data $D_0$, policy confidence threshold $\mu$.
  \Ensure A trained policy $\hat{\pi}$.
  \State Initialize $D \leftarrow D_0$.
  \State Initialize $\hat{\pi}_1$ by training it on $D_0$.
  \For{$i=1$ to $N$}
    \State Observe $m$ user questions $q_j, j \in [1,m]$;
    \State $D_i \leftarrow \bigcup_{j=1}^m \textsc{Parse\&Collect}(\mu, q_j, \hat{\pi}_i, \pi^*)$;
    \State Aggregate dataset $D \leftarrow D \bigcup D_i$;
    \State Train policy $\hat{\pi}_{i+1}$ on $D$ using Eq.~\eqref{eq:general}.
  \EndFor
  \State \textbf{return} best $\hat{\pi}_i$ on validation.
  \vspace{2mm}
  \Function{Parse\&Collect}{$\mu, q, \hat{\pi}_i, \pi^*$}
    \State Initialize $D_i^\prime \leftarrow \emptyset$, $s_1=(q, \phi)$.
    \For{$t=1$ to $T$}
        \State Preview action $a_t = \argmax_{a} \hat{\pi}_i(s_t)$; 
        \If{$p_{\hat{\pi}_i}(a_t|s_t) \geq \mu$}
            \State $w_t \leftarrow 1$;
            \State Collect $D_i^\prime \leftarrow D_i^\prime \bigcup \{(s_t, a_t, w_t)\}$;
            \State Execute $a_t$;
        \Else
            \State Trigger user interaction and derive expert demonstration $a_t^* \sim \pi^*(s_t)$;
            \State $w_t \leftarrow 1$ if $a_t^*$ is valid; $0$ otherwise;
            \State Collect $D_i^\prime \leftarrow D_i^\prime \bigcup \{(s_t, a_t^*, w_t)\}$;
            \State Execute $a_t^*$.
        \EndIf
    \EndFor
    \State \textbf{return} $D_i^\prime$. 
  \EndFunction
  \end{algorithmic}
  \caption{The \neil Algorithm}
  \label{alg:mispl_general}
\end{algorithm}

\subsection{Annotation-efficient Imitation Learning}

Consider parsing a user question and collecting training data using the parser $\hat{\pi}_i$ in the $i$-th iteration (Line 5). A standard imitation learning algorithm such as \textsc{DAgger} \cite{ross2011reduction} usually requests expert demonstration $a_t^*$ for every state $s_t$ in the sampled trajectory. However, it requires a considerable amount of user annotations, which may not be practical when interacting with end users.

Instead, we propose to adopt an \emph{annotation-efficient} learning strategy in \neil, which saves user annotations by \emph{selectively} requesting user interactions, as indicated in function \textsc{Parse\&Collect}. In each parsing step, the system first previews whether it is confident about its own decision $a_t$ (Line 13--14), which is determined when its probability is no less than a threshold, i.e., $p_{\hat{\pi}_i}(a_t|s_t) \geq \mu$.\footnote{{The metric is shown effective for interactive semantic parsing in \citet{yao2019model}. Other confidence measures can also be explored, as we will discuss in Section~\ref{sec:future_work}.}}
In this case, the algorithm executes and collects \nop{the policy action}{its own action} $a_t$ (Line 15--17); otherwise, a system-user interaction will be triggered and the derived demonstration $a_t^* \sim \pi^*(s_t)$ will be collected and executed to continue parsing (Line 19--22).

Denote a collected state-action pair as $(s_t, \Tilde{a}_t)$, where $\Tilde{a}_t$ could be $a_t$ or $a_t^*$ depending on whether an interaction is requested. To train $\hat{\pi}_{i+1}$ (Line 7), our algorithm adopts a \emph{reduction-based} approach similar to \textsc{DAgger} and reduces imitation learning to iterative supervised learning. Formally, we define our training loss function as a \emph{weighted} negative log-likelihood:
\vspace{-2mm}
\begin{equation}\label{eq:general}
\small
    \mathcal{L}(\hat{\pi}_{i+1}) = - \frac{1}{|D|} \sum_{(s_t, \Tilde{a}_t, w_t) \in D} w_t \log p_{\hat{\pi}_i}(\Tilde{a}_t|s_t),
    \vspace{-2mm}
\end{equation}
where $D$ is the aggregated training data over $i$ iterations and $w_t$ denotes the weight of $(s_t, \Tilde{a}_t)$.

We consider assigning weight $w_t$ in three cases: 
(1) For confident actions $a_t$, we set $w_t=1$. This essentially treats \nop{confident actions}{the system's own confident actions} as gold decisions, which resembles self-training \cite{scudder1965probability,nigam2000analyzing, mcclosky2006effective}. 
(2) For user-confirmed decisions (\emph{valid demonstrations} $a_t^*$), such as enforcing a \texttt{WHERE} condition on ``\texttt{Player}'' in Figure~\ref{fig:intro}, $w_t$ is also set to 1 to encourage the parser to imitate the correct decisions from users. 
(3) For uncertain actions that cannot be addressed via human interactions (\emph{invalid demonstrations} $a_t^*$, which are identified when the user selects ``None of the above options'' in Figure~\ref{fig:intro}), we assign $w_t=0$. This could happen when some of the incorrect precedent actions are not fixed. For example, in Figure~\ref{fig:intro}, if the system missed correcting the \texttt{WHERE} condition on ``\texttt{School/Club Team}'', then whatever value it generates after ``\texttt{WHERE School/Club Team=}'' is wrong, and thus any action $a_t^*$ derived from human feedback would be invalid. {In this case, the system selects the next available option without further validation and continues parsing.}

A possible training strategy to handle case (3)\nop{in this case} may set $w_t$ to be negative, similar to \citet{Welleck2020Neural}. However, empirically we find this strategy fails to train the parser to correct its mistake in generating ``\texttt{School/Club Team}'' but rather disturbs the model training. 
By setting $w_t=0$, the impact of unaddressed actions is removed from training.
A similar solution is also adopted in \citet{petrushkov2018learning, kreutzer2019self}. As shown in Section~\ref{sec:experiments}, this way of training weight assignment enables stable improvement in iterative model learning while requiring fewer user annotations.

\section{Theoretical Analysis} \label{sec:theoretical_analysis}

While \neil enjoys the benefit of learning from a small amount of user feedback, one crucial question is whether it can still achieve the same level of performance as the traditional supervised approach (which trains a policy on full expert annotations, if one could afford that and manage the privacy risk). In this section, we prove that the performance gap between the two approaches is mainly determined by the learning policy's probability of trusting a confident action that turns out to be wrong, which can be controlled in practice.

Our analysis follows prior work ~\cite{ross2010efficient,ross2011reduction} to assume a unified trajectory length $T$ and an infinite number of training samples in each iteration (i.e., $m=\infty$ in Algorithm~\ref{alg:mispl_general}), such that the state space can be full explored by the learning policy. An analysis under the ``finite sample'' case can be found in Appendix~\ref{app:subsec:finite}.

\subsection{Cost Function for Analysis}

Unlike typical imitation learning tasks (e.g., Super Tux Kart \cite{ross2011reduction}), in semantic parsing, there exists only one gold trajectory semantically identical to the question.\footnote{We assume a canonical order for swappable components in a parse. In practice, it may be possible, though rare, for one question to have multiple gold parses.} Whenever a policy action is different from the gold one, the whole trajectory will not yield the correct semantic meaning and the parsing is deemed failed. {In other words, a well-performing semantic parser should be able to keep staying in the correct trajectory during the parsing.} Therefore, for theoretical analysis, we only analyze a policy's performance when it is conditioned on a \emph{gold partial parse}, i.e., $s_t \in d_{\pi^*}^t$, where $d_{\pi^*}^t$ is the state distribution in step $t$ when executing the expert policy $\pi^*$ for first $t$-1 steps. 
Let $\ell(s, \hat{\pi}) = 1 - p_{\hat{\pi}}(a = a^*|s)$ be the loss of $\hat{\pi}$ making a mistake at state $s$.
We define the \emph{cost} (i.e., the inverse \emph{test-time} quality) of a policy as:
\begin{equation}
    J(\hat{\pi}) = T \mathbb{E}_{s \sim d_{\pi^*}}\big[\ell(s,\hat{\pi}) \big],\label{eqn:3}
    \vspace{-2mm}
\end{equation}
where $d_{\pi^*} = \frac{1}{T}\sum_{t=1}^T d_{\pi^*}^t$ denotes the average expert state distribution (assuming time step $t$ is a random variable uniformly sampled from $1 \sim T$). A detailed derivation is shown in Appendix~\ref{app:cost_function}.

{The \emph{better} a policy $\hat{\pi}$ is, the \emph{smaller} this cost becomes. Note that, by defining Eq.~\eqref{eqn:3}, we simplify the analysis from evaluating \emph{the whole trajectory sampled from $\hat{\pi}$} (as we do in experiments) to evaluating \emph{the expected single-step loss of $\hat{\pi}$ conditioned on a gold partial parse}. This cost function makes the analysis easier and meanwhile reflects a consistent relative performance among algorithms for comparison. 
Next, we compare our \neil algorithm with the supervised approach by analyzing the upper bounds of their costs.
}

\subsection{Cost Bound of Supervised Approach}
A fully supervised system trains a parser on expert-annotated $(q, a_{1:T}^*)$ pairs, where the gold semantic parse $a_{1:T}^*$ can be viewed as generated by executing the expert policy $\pi^*$. This gives the policy $\hat{\pi}_{sup}$:
\begin{equation*}
    \hat{\pi}_{sup} = \argmin_{\pi \in \Pi} \mathbb{E}_{s \sim d_{\pi^*}}[l(s, \pi)],
    \vspace{-2mm}
\end{equation*}
where $\Pi$ is the policy space induced by the model architecture. A detailed derivation in Appendix~\ref{app:sup_cost_bound} shows the cost bound of the supervised approach:
\begin{theorem}\label{theorem:sup}
    For supervised approach, let $\epsilon_N = \min_{\pi \in \Pi} \mathbb{E}_{s \sim d_{\pi^*}}[l(s, \pi)]$, then $J(\hat{\pi}_{sup}) = T \epsilon_N$.
\end{theorem}
The theorem gives an exact bound (as shown by the equality) since the supervised approach, given the ``infinite sample'' assumption, trains a policy under the same state distribution $d_{\pi^*}$ as the one being evaluated in the cost function (Eq.~\eqref{eqn:3}).

\subsection{Cost Bound of \neil Algorithm}
\label{subsec:obj_bound}

Recall that, in each training iteration, \neil samples trajectories by executing actions from both the previously learned policy $\hat{\pi}_i$ and the expert policy $\pi^*$ (when an interaction is requested). Let $\pi_i$ denote such a ``mixture'' policy. We derive the following cost bound of a \neil policy $\hat{\pi}$:
\begin{equation}
    J(\hat{\pi}) \leq \frac{T}{N} \sum_{i=1}^N \big[ \mathbb{E}_{s \sim d_{\pi_i}}[\ell(s, \hat{\pi}_i)] + \ell_{max} ||d_{\pi_i} - d_{\pi^*}||_1 \big].\nonumber
\end{equation}
The bound is determined by two terms. The first term $\mathbb{E}_{s \sim d_{\pi_i}}[\ell(s, \hat{\pi}_i)]$ calculates the expected training loss of $\hat{\pi}_i$. Notice that, while the policy is \emph{trained} on states induced by the mixture policy ($s \sim d_{\pi_i}$), what matters to its \emph{test-time} quality is the policy's performance conditioned on a gold partial parse ($s \sim d_{\pi^*}$ in Eq.~\eqref{eqn:3}). 
This state discrepancy, which does not exist in the supervised approach, explains the performance loss of \neil, and is bounded by the second term $\ell_{max} ||d_{\pi_i} - d_{\pi^*}||_1$, the weighted $L_1$ distance between $d_{\pi_i}$ and $d_{\pi^*}$. 
{To bound the two terms, we employ a ``no-regret'' assumption \cite[see Appendix~\ref{app:no-regret}--\ref{app:mispl_cost_bound} for details]{kakade2009generalization, ross2011reduction}, which gives the theorem:}
\begin{theorem}\label{theorem:mispl}
    For the proposed \neil algorithm, if $N$ is $\Tilde{O}(T)$, there exists a policy $\hat{\pi} \in \hat{\pi}_{1:N}$ s.t. $J(\hat{\pi}) \leq T \big[ \epsilon_N + \frac{2T \ell_{max}}{N} \sum_{i=1}^N e_i \big] + O(1)$.
\end{theorem}
Here, $\epsilon_N = \min_{\pi \in \Pi} \frac{1}{N} \sum_{i=1}^N \mathbb{E}_{s \sim d_{\pi_i}}[\ell(s, \pi)]$ denotes the best expected policy loss in \emph{hindsight}, and $e_i$ denotes the probability that $\hat{\pi}_i$ does not query the expert policy (i.e., being confident) but its own action is wrong under $d_{\pi^*}$.

{We note that a no-regret algorithm requires convexity of the loss function \cite{hazan2007logarithmic, kakade2009generalization}, which is not satisfied by neural network-based semantic parsers. 
In general, proving theorems under a non-convex case is not trivial. Therefore, we follow the common {practice}\nop{routine} (e.g., \citet{DBLP:journals/corr/KingmaB14, j.2018on}) to theoretically analyze the convex case while empirically demonstrating the performance of our \neil algorithm with non-convex loss functions (i.e., when it applies to neural semantic parsers). More accurate regret bound for non-convex cases will be studied in the future.}

\begin{remark}[Remarks.]
{Compared with the supervised approach (Theorem~\ref{theorem:sup}), \neil's cost bound additionally contains a term of $\frac{1}{N}\sum_{i=1}^N e_i$, which, as we expect, comes from the aforementioned state discrepancy. Intuitively, if a learning policy frequently executes its own but wrong actions in training, the resulting training states $d_{\pi_i}$ will greatly deviate from the gold ones $d_{\pi^*}$. }

This finding inspires us to restrict the performance gap by reducing the learning policy's error rate \emph{when it does not query the expert}. Empirically this can be achieved by: 
(1) \emph{Accurate confidence estimation}, so that actions deemed confident are generally correct, and (2) \emph{Moderate policy initialization}, such that in general the policy is less likely to make wrong actions throughout the iterative training.
For (1), we set a high confidence threshold $\mu$=0.95, which is demonstrated reliable for MISP \cite{yao2019model}. We then empirically validate (2) in experiments.
\end{remark}

\section{Experiments}\label{sec:experiments}
In this section, we conduct experiments to demonstrate the annotation efficiency of our \neil algorithm and that it can train semantic parsers for high performance when the parsers are reasonably initialized, which verifies our theoretical analysis.

\subsection{Experimental Setup}
We compare various systems on the WikiSQL dataset \cite{zhong2017seq2sql}. The dataset contains large-scale annotated question-SQL pairs (56,355 pairs for training) and thus serves as a good resource for experimenting iterative learning. For the base semantic parser, we choose SQLova \cite{hwang2019comprehensive}, one of the top-performing models on WikiSQL, to ensure a reasonable model capacity in terms of data utility along iterative training.

We experiment each system with three parser initialization settings, using 10\%, 5\% and 1\% of the total training data. During iterative learning, questions from the remaining training data arrive in a random order to simulate user questions and we simulate user feedback by directly comparing the synthesized query with the gold one. 
{In each iteration, all systems access exactly the same user questions. Depending on how they solicit feedback, each system collects a different number of annotations. At the end of each iteration, we update each system by retraining its parser on its accumulated annotations and the initial training data, and report its \emph{(exact) query match accuracy} on the test set. We also report the \emph{accumulated number of annotations} that each system has requested after each training iteration, in order to compare their annotation efficiency.} 

In experiments, we consider every 1,000 user questions as one training iteration (i.e., $m$=1,000 in Algorithm~\ref{alg:mispl_general}). We repeat the whole iterative training for three runs and report average results.
\emph{Reproducible details} are included in Appendix~\ref{app:implementation_details}.

\begin{figure*}[t!]
    \centering
    \includegraphics[width=0.95\linewidth]{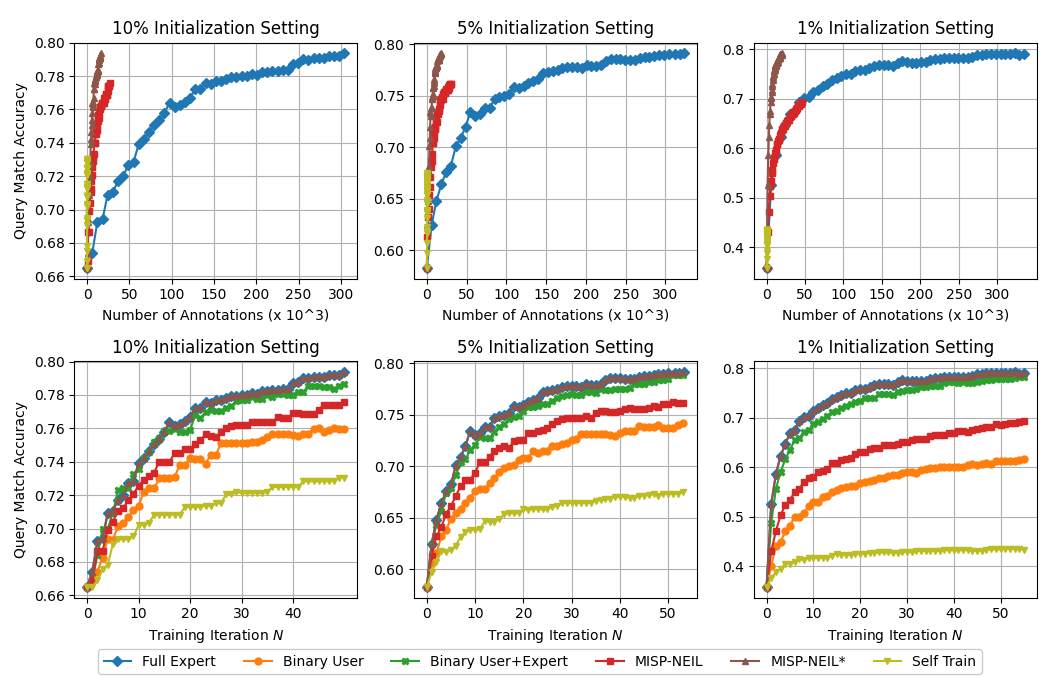}
    \caption{Parsing accuracy on WikiSQL test set when systems are trained with various numbers of user/expert annotations (top) and for different iterations (bottom). We experiment with three initialization settings, using 10\%, 5\% and 1\% of the training data respectively. Results on validation set can be found in Appendix~\ref{app:detailed_exp}.}
    \label{fig:exp}
\end{figure*}

\subsection{System Comparison}
We denote our system as \textbf{\mispl} since it leverages MISP in the back end of \neil.
We compare it with the traditional supervised approach (denoted as \textbf{Full Expert}). To investigate the skyline capability of our system, we also present a variant called \textbf{\mispls}, which is \emph{assumed} with perfect confidence measurement and interaction design, so that it can precisely identify and correct its mistakes during parsing. This is implemented by allowing the system to compare its synthesized query with the gold one. Note that this is not a realized automatic system; we show its performance as an upper bound of \mispl.

On the other hand, although execution feedback-based learning systems \cite{clarke2010driving, artzi2013weakly, iyer2017learning} may not be very practical for end users, we include them nonetheless in the interest of a comprehensive comparison. This leads to two baselines. The \textbf{Binary User} system requests binary user feedback on whether \emph{executing} the generated SQL query returns correct database results and collects only queries with correct results to further improve the parser. The \textbf{Binary User+Expert} system additionally collects full expert SQL annotations when the generated SQL queries do not yield correct answers.

{Given the completely different nature of annotations from Binary User (which validate the \emph{denotation}) and those from Full Expert and \mispl (which validate a semantic parse's \emph{constituents}), there may not exist a universally fair way to convert one's annotation consumption into the other's. Therefore, in the following sections, we only present and discuss Binary User(+Expert) in terms of their parsing accuracy under different training iterations. 
To give an estimation of their annotation efficiency for reference, we design a compromised annotation calculation metric for Binary User(+Expert) and include their results on WikiSQL validation set in Appendix~\ref{app:detailed_exp}. 
}

{Finally, while our \mispl and the aforementioned baselines all leverage feedback from users or domain experts, an interesting question is how much gain they could obtain compared with using no annotation or feedback at all. To this end, we compare the systems with a \textbf{Self Train} baseline \cite{scudder1965probability, nigam2000analyzing, mcclosky2006effective}. 
In each iteration, this baseline collects SQL queries generated \emph{by itself} as the new gold annotations for further training. We additionally apply a confidence threshold to improve the collection quality, i.e., only SQL queries with probability $p_{\hat{\pi}}(a_{1:T}|s_1)$ greater than 0.5 are included. This strategy empirically leads to better performance. Intuitively, we expect Self Train to perform no better than any other systems in our experiments, since no human feedback is provided to correct mistakes in its collection.}

\subsection{Experimental Results}\label{subsec:exp_results}

We evaluate each system by answering two research questions (RQs):
\begin{itemize}[noitemsep,topsep=1pt]
    \item \emph{RQ1: Can the system improve a parser without requiring a large amount of annotations?} 
    \item \emph{RQ2: For interactive systems, while requiring weaker supervision, can they train the parser to reach a performance comparable to the traditional supervised system?}
\end{itemize}

For \textit{RQ1}, we measure the number of \emph{user/expert} annotations that a system requires to train a parser. For Full Expert, this number equals the trajectory length of the gold query (e.g., 5 for the query in Figure~\ref{fig:intro}); for \mispl and \mispls, it is the number of user interactions during training. 
Note that while we do not differentiate the actual (e.g., time/financial) cost of users from that of experts in this aspect, we emphasize that our system enjoys an additional benefit of collecting training examples from a much cheaper and more abundant source. {For Self Train, the number of annotations is always zero since it does not request any human feedback for the online user questions.}

Our results in Figure~\ref{fig:exp} (top) demonstrate that \mispl consistently consumes a comparable or smaller amount of annotations to train the parser to reach the same parsing accuracy. 
Figure~\ref{fig:q_curve} in Appendix further shows that, on average, it requires no more than one interaction for each user question along the training. Particularly in the 10\% initialization setting, \mispl uses less than 10\% of the total annotations that Full Expert needs in the end. Given the limited size of WikiSQL training set, the simulation experiments currently can only show \mispl's performance under a small number of annotations. However, we expect this gain to continue as it receives more user questions in the long-term deployment.

To answer \textit{RQ2}, Figure~\ref{fig:exp} (bottom) compares each system's accuracy after they have been trained for the same number of iterations. The results demonstrate that when a semantic parser is moderately initialized (10\%/5\% initialization setting), \mispl can further improve it to reach a comparable accuracy as Full Expert (0.776/0.761 vs. 0.794 in the last iteration). 
In the extremely weak 1\% initialization setting (using only around 500 initial training examples), all interactive learning systems suffer from a huge performance loss. This is consistent with our finding in theoretical analysis (Section~\ref{sec:theoretical_analysis}). In Appendix~\ref{app:subsec:theoretical_exp}, we plot the value of $e_i$, the probability that $\hat{\pi}_i$ makes a confident but wrong decision given a gold partial parse, showing that a better initialized policy generally obtains a smaller $e_i$ throughout the training and thus a tighter cost bound.

Our system also surpasses Binary User. We find that the inferior performance of Binary User is mainly due to the ``spurious program'' issue \cite{guu2017language}, i.e., a SQL query having correct execution results can still be incorrect in terms of semantics.\nop{For example, ``\texttt{WHERE C1=A}'' and ``\texttt{WHERE C1=A and C2=B}'' may give the same execution results when all database records satisfying ``\texttt{C1=A}'' also meet ``\texttt{C2=B}'' by accident. However, semantically the latter includes an extra condition which may not be specified by the question.} \mispl circumvents this issue by directly validating the \emph{semantic meaning} of intermediate parsing decisions. {The performance of Binary User+Expert is close to Full Expert as it has additionally involved expert annotations on a considerable number of user questions, which on the other hand also leads to extra annotation overhead.}

When it is assumed with perfect interaction design and confidence estimator, \mispls shows striking superiority in both aspects. Since it always corrects wrong decisions immediately, \mispls can collect and derive the same training examples as Full Expert, and thus trains the parser to Full Expert's performance level in Figure~\ref{fig:exp} (bottom). However, it requires only 6\% of the annotations that Full Expert needs (Figure~\ref{fig:exp}, top). These observations imply large room for \mispl to be improved in the future. 

{Finally, we observe that all feedback-based learning systems outperform Self Train dramatically (Figure~\ref{fig:exp}, bottom). This verifies the benefit of {learning from human feedback}.}

\subsection{Generalize to Complex SQL Queries}\label{subsec:spider}
We next investigate whether \mispl can generalize to the complex SQL queries in the Spider dataset \cite{yu2018spider}, which can contain complicated keywords like \texttt{GROUP BY}.
For the base semantic parser, we choose EditSQL \cite{zhang2019editing}, one of the open-sourced top models on Spider. Given the small size of Spider (7,377 question-SQL pairs for training after data cleaning; see Appendix~\ref{app:subsec:editsql} for details), we only experiment with one initialization setting, using 10\% of the training set.
Since \nop{all Spider models}EditSQL does not predict the specific values in a SQL query (e.g., ``\texttt{jalen rose}'' in Figure~\ref{fig:intro}), we cannot execute the generated query to simulate the binary execution feedback. Therefore, we only compare our system with Full Expert and Self Train. 
Parsers are evaluated on Spider Dev set since its test set is not publicly available.

\begin{figure}
    \centering
    \includegraphics[width=0.9\linewidth]{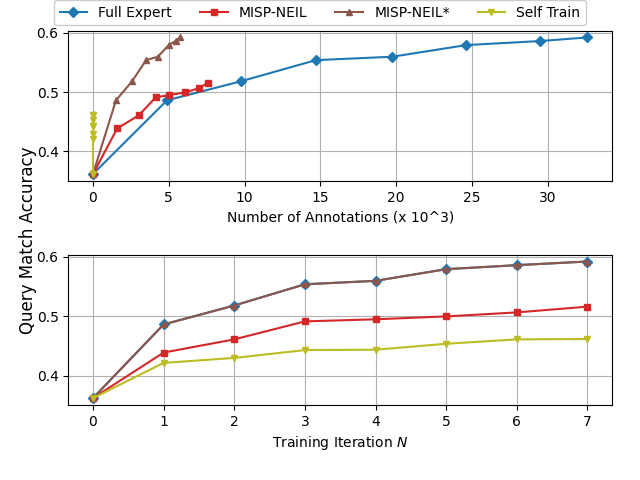}
    \caption{Parsing accuracy on Spider Dev set when systems are trained with various numbers of user/expert annotations and for different iterations.}
    \label{fig:editsql}
\vspace{-4mm}
\end{figure}

Figure~\ref{fig:editsql} (top) shows that \mispl and \mispls consistently achieve comparable or better annotation efficiency while enjoying the advantage of learning from end user interaction. We expect this superiority to continue as the systems receive more user questions beyond Spider. Meanwhile, we also notice that the gain is smaller and \mispl suffers from a large performance loss compared with Full Expert (Figure~\ref{fig:editsql}, bottom), due to the poor parser initialization and the SQL query complexity. This can be addressed via adopting better interaction designs and more accurate confidence estimation, as shown by \mispls. {Similarly as in WikiSQL experiments, Self Train performs worse than human-in-the-loop learning systems, as there is no means to correct wrong predictions in its collected annotations.}

\section{Conclusion and Future Work}\label{sec:future_work}

Our work shows the possibility of continually learning semantic parsers from fine-grained end user interaction. As a pilot study, we experiment systems with simulated user interaction. One important future work is thus to conduct large-scale user studies and train parsers from real user interaction. This is not trivial and has to account for uncertainties such as noisy user feedback. We also plan to derive a more realistic formulation of user/expert annotation costs by analyzing real user statistics (e.g., average time spent on each question).

In experiments, we observe that neural semantic parsers tend to be overconfident and training them with more data does not mitigate this issue. In the future, we will look into more accurate confidence measure via neural network calibration \cite{guo2017calibration} or using machine learning components (e.g., answer triggering \cite{zhao2017end} or a reinforced active selector \cite{fang2017learning}).

Finally, we believe our algorithm can be applied to save annotation effort for other NLP tasks, especially the low-resource ones \cite{mayhew2019named}.

\section*{Acknowledgments}
We would like to thank Khanh Nguyen and the anonymous reviewers for their helpful discussions or comments.
{The research conducted by the Ohio State University authors was sponsored in part by the Army Research Office under cooperative agreements W911NF-17-1-0412, NSF Grant IIS1815674, NSF CAREER \#1942980, Fujitsu gift grant, and Ohio Supercomputer Center \cite{OhioSupercomputerCenter1987}. The views and conclusions contained herein are those of the authors and should not be interpreted as representing the official policies, either expressed or implied, of the Army Research Office or the U.S. Government. The U.S. Government is authorized to reproduce and distribute reprints for Government purposes notwithstanding any copyright notice herein.}
The Spider dataset~\cite{yu2018spider} is distributed under the CC BY-SA 4.0 license.

\bibliography{emnlp2020} 
\bibliographystyle{acl_natbib}

\clearpage
\appendix

\section{Theoretical Analysis in Infinite Sample Case}\label{app:infinite}

In this section, we give a detailed theoretical analysis to derive the cost bounds of the supervised approach and our proposed \neil algorithm (Section~\ref{sec:model}). 
Following \citet{ross2011reduction}, we first focus the proof on an \emph{infinite} sample case, which assumes an infinite number of samples to train a policy in each iteration (i.e., $m=\infty$ in Algorithm~\ref{alg:mispl_general}), such that the state space in training can be full explored by the learning policy.

As an overview, we start the analysis by introducing the ``cost function'' we use to analyze each policy in Appendix~\ref{app:cost_function}, which represents an inverse quality of a policy. In Appendix~\ref{app:sup_cost_bound}, we derive the bound of the cost of the supervised approach. Appendix~\ref{app:no-regret} and Appendix~\ref{app:mispl_cost_bound} then discuss the cost bound of our proposed \neil algorithm. Finally, in Appendix~\ref{app:subsec:finite}, we show the cost bound of \neil in \emph{finite} sample case.

\subsection{Cost Function for Analysis}\label{app:cost_function}
In a semantic parsing task, whenever a policy action is different from the gold one, the whole trajectory cannot yield the correct semantic meaning and the parsing is deemed failed. Therefore, we analyze a policy's performance only when it is conditioned on \emph{a gold partial parse}. Intuitively, a policy with better quality should have a higher parsing accuracy under a gold partial parse, so that it is more likely to sample a completely correct trajectory in inference time.

Given a question $q$ and denoting $a_{1:t}^*$ as the gold partial trajectory sampled by the expert policy $\pi^*$, we first define the \emph{cost} of sampling a partial trajectory $a_{1:t}=(a_1, ..., a_t)$ as:
\begin{equation*}
C(q, a_{1:t}) = 
\begin{cases}
	0 & \text{if}\;\; a_{1:t} = a_{1:t}^* \\
    1 & \text{otherwise}
\end{cases}.
\end{equation*}
In other words, a sampled partial trajectory is correct if and only if it is the same as the gold partial parse.
Based on this definition, we further define the expected cost of policy $\hat{\pi}$ \emph{in a single time step $t$} (given the question $q$) as:
\begin{align*}
    C_{\hat{\pi}}^t(q) &= \mathbb{E}_{a_{1:t-1} \sim \pi^*} \mathbb{E}_{a_t\sim \hat{\pi}} [C(q, a_{1:t})] \\
    &= \mathbb{E}_{a_{1:t-1} \sim \pi^*}[1 - p_{\hat{\pi}}(a_t=a_t^*|q, a_{1:t-1})].
\end{align*}
Here, $a_{1:t-1} \sim \pi^*$ denotes a gold partial parse till the ($t$-1)-th step, which is obtained by executing the expert policy $\pi^*$ for the first $t$-1 steps (given $q$), and $p_{\hat{\pi}}(a_t=a_t^*|q, a_{1:t-1})$ denotes the probability that $\hat{\pi}$ samples action $a_t^*$ given a question $q$ and a partial parse $a_{1:t-1}$. By taking an expectation over all questions $q \in \mathcal{Q}$, we have the following derivations:
\begin{align}
\scriptstyle
\mathbb{E}_{q \in \mathcal{Q}} [C_{\hat{\pi}}^t(q)]
&= \scriptstyle \mathbb{E}_{q \in \mathcal{Q}, a_{1:t-1} \sim \pi^*} [1 - p_{\hat{\pi}}(a_t=a_t^*|q, a_{1:t-1})]  \nonumber\\
&= \scriptstyle \mathbb{E}_{s_t \sim d_{\pi^*}^t}[1 - p_{\hat{\pi}}(a_t=a_t^*|s_t)].\nonumber
\end{align}
The second equality holds by the definition $s_t=(q,a_{1:t-1})$, and $d_{\pi^*}^t$ is the ``expert state distribution'' in step $t$ when executing the expert policy $\pi^*$ for first $t$-1 steps. In this analysis, we follow \citet{ross2010efficient, ross2011reduction} to assume a unified decision length $T$. By summing up the above expected cost over the $T$ steps, we define the \emph{cost} (i.e., the inverse \emph{test-time} quality) of policy $\hat{\pi}$:
\begin{align*}
\vspace{-2mm}
    J(\hat{\pi}) &= \sum_{t=1}^T \mathbb{E}_{q \in \mathcal{Q}} [C_{\hat{\pi}}^t(q)] \\
    &=  \sum_{t=1}^T \mathbb{E}_{s_t \sim d_{\pi^*}^t}[1 - p_{\hat{\pi}}(a_t=a_t^*|s_t)].
\end{align*}
Denote $\ell(s, \hat{\pi}) = 1 - p_{\hat{\pi}}(a=a^*|s), a \sim \hat{\pi}(s), a^* \sim \pi^*(s)$ as the ``single-step loss function'', which is bounded within $[0,1]$, then the cost of policy $\hat{\pi}$ can be simplified as:
\begin{align}
    J(\hat{\pi}) &= \sum_{t=1}^T \mathbb{E}_{s_t \sim d_{\pi^*}^t} \big[ \ell(s_t,\hat{\pi}) \big] \nonumber\\
    &= T\mathbb{E}_{t \sim \mathcal{U}(1,T)} \mathbb{E}_{s_t \sim d_{\pi^*}^t} \big[ \ell(s_t,\hat{\pi}) \big]\nonumber\\
    &= T \mathbb{E}_{s \sim d_{\pi^*}}\big[\ell(s,\hat{\pi}) \big],\label{eqn:app:obj}
\end{align}
where $d_{\pi^*} = \frac{1}{T}\sum_{t=1}^T d_{\pi^*}^t$ is the average expert state distribution, when we assume the time step $t$ to be a random variable under the uniform distribution $\mathcal{U}(1,T)$ (the second equality).

\subsection{Cost Bound of Supervised Approach}\label{app:sup_cost_bound}
In this section, we analyze the cost bound of the supervised approach.
Recall that the supervised approach trains a policy $\hat{\pi}$ using the standard supervised learning algorithm with supervision from $\pi^*$ at every decision step. Therefore, it finds the best policy $\hat{\pi}_{sup}$ on infinite samples as:
\begin{equation}\label{eqn:app:sup_obj}
    \hat{\pi}_{sup} = \argmin_{\pi \in \Pi} \mathbb{E}_{s \sim d_{\pi^*}}[\ell(s, \pi)],
\end{equation}
where $\Pi$ denotes the policy space induced by the model architecture, and the expectation over $s$ is sampled from the whole $d_{\pi^*}$ state space because of the ``infinite sample'' assumption. The supervised approach thus obtains the following cost bound:
\begin{equation*}
\begin{split}
    J(\hat{\pi}_{sup}) =& T \mathbb{E}_{s \sim d_{\pi^*}}[\ell(s, \hat{\pi}_{sup})]\\ = & T \min_{\pi \in \Pi}\mathbb{E}_{s \sim d_{\pi^*}}[\ell(s, \pi)].
\end{split}
\end{equation*}
This gives the following theorem:
\begin{theorem}\label{theorem:app:sup}
    For supervised approach, let $\epsilon_N = \min_{\pi \in \Pi} \mathbb{E}_{s \sim d_{\pi^*}}[\ell(s, \pi)]$, then $J(\hat{\pi}_{sup}) = T \epsilon_N$.
\end{theorem}

The cost bound of the supervised approach represents its exact performance as implied by the equality. This is because the approach trains a policy (Eq.~\eqref{eqn:app:sup_obj}) under the same state distribution $d_{\pi^*}$ (given the ``infinite sample'' assumption) as in evaluation (Eq.~\eqref{eqn:app:obj}).
As we will show next, the proposed \neil algorithm breaks this consistency while enjoying the benefit of high annotation efficiency, which explains the performance gap.

\subsection{No-regret Assumption}\label{app:no-regret}
Before showing the cost bound of our \neil algorithm, we introduce a ``no-regret'' assumption \cite{kakade2009generalization, ross2011reduction} that is leveraged in the derivation.

\begin{assumption}\label{assumption:no-regret}
    No-regret assumption. Define $\ell_i(\pi) = \mathbb{E}_{s \sim d_{\pi_i}}[l(s, \pi)]$ and $\epsilon_N = \min_{\pi \in \Pi} \frac{1}{N} \sum_{i=1}^N \ell_i(\pi)$, then
    $$
    \frac{1}{N}\sum_{i=1}^N \ell_i(\hat{\pi}_i) - \epsilon_N \leq \gamma_N
    $$
    for $\lim_{N \rightarrow \infty} \gamma_N = 0$ (usually $\gamma_N \in \Tilde{O}(\frac{1}{N})$).
\end{assumption}
This assumption characterizes an important policy learning pattern: As a policy is trained for an infinite number of iterations, on average, its expected training loss ($\frac{1}{N}\sum_{i=1}^N \ell_i(\hat{\pi}_i)$) will converge to the loss of the best policy in hindsight ($\epsilon_N$). In our scenario, this assumption implies that, while our policy is trained on online labels from both the expert policy $\pi^*$ (when it is queried) and the previously learned policy $\hat{\pi}_i$ (when the agent is confident), it still gradually fits to the best policy over the same state space in training ($d_{\pi_i}$). In other words, the likely noisy labels from $\hat{\pi}_i$ do not harm the model fitting to the expert policy in general.

Many no-regret algorithms \cite{hazan2007logarithmic,kakade2009generalization} that guarantee $\gamma_N \in \Tilde{O}(\frac{1}{N})$ require convexity or strong-convexity of the loss function. However, the loss function used in our application, which is built on the top of a deep neural network model, does not satisfy this requirement. 
{In general, proving theorems under a non-convex case is not trivial. In this analysis, we follow the common practice (see \citet{DBLP:journals/corr/KingmaB14,j.2018on} for example) to theoretically analyze the convex case while empirically demonstrating the non-convex case.}
A more accurate regret bound for non-convex neural networks (which may result in a slower $\gamma_N$ convergence speed with respect to $N$) can be studied in the future.

\subsection{Cost Bound of \neil Algorithm}\label{app:mispl_cost_bound}

As shown in Algorithm~\ref{alg:mispl_general}, \neil produces a sequence of policies $\hat{\pi}_{1:N} = (\hat{\pi}_1, \hat{\pi}_2, ..., \hat{\pi}_N)$, where $N$ is the number of training iterations, and returns the one with the best test-time performance on validation set as $\hat{\pi}$. In training, the algorithm executes actions from both the learning policy $\hat{\pi}_i$ (when the model is confident) and the expert policy $\pi^*$. We denote this ``mixture'' policy as $\pi_i$. Then for the first $N$ iterations, we have the cost bound of \neil as:
\begin{align}\label{eqn:app:mispl_derivation}
    \scriptstyle
    J(\hat{\pi}) =& \scriptstyle \min_{\hat{\pi}^\prime \in \hat{\pi}_{1:N}} T \mathbb{E}_{s \sim d_{\pi^*}} \big[\ell(s, \hat{\pi}^\prime)\big] \nonumber\\
    \leq & \scriptstyle \frac{T}{N} \sum_{i=1}^N \mathbb{E}_{s \sim d_{\pi^*}} \big[\ell(s, \hat{\pi}_i)\big] \nonumber\\
    \leq & \scriptstyle \frac{T}{N} \sum_{i=1}^N \big[ \mathbb{E}_{s \sim d_{\pi_i}}[\ell(s, \hat{\pi}_i)] + \ell_{max} ||d_{\pi_i} - d_{\pi^*}||_1 \big].
\end{align}

From the last inequality, we can see that the cost bound of \neil is restricted by two terms. The first term $\mathbb{E}_{s \sim d_{\pi_i}}[\ell(s, \hat{\pi}_i)]$ denotes the expected loss of $\hat{\pi}_i$ under the states induced by $\pi_i$ \emph{during training} (under the ``infinite sample'' assumption, as mentioned in the beginning of the analysis). 
By applying the no-regret assumption (Assumption~\ref{assumption:no-regret}), this term can be bound by $\frac{1}{N}\sum_{i=1}^N \mathbb{E}_{s \sim d_{\pi_i}}[\ell(s, \hat{\pi}_i)] \leq \epsilon_N + \gamma_N$. Here, $\epsilon_N = \min_{\pi \in \Pi} \frac{1}{N} \sum_{i=1}^N \ell_i(\pi)$ denotes the best expected training loss in hindsight.

The second term denotes the $L_1$ distance between state distributions induced by $\pi_i$ and $\pi_i^*$, weighted by the maximum loss value $l_{max}$ that $\hat{\pi}_i$ encounters over the training. 
As we notice, unlike the supervised approach, \neil \emph{trains} a policy under $d_{\pi_i}$, while what matters to its \emph{test-time} quality is its performance on the state distribution $d_{\pi^*}$ (Eq.~\eqref{eqn:app:obj}). This discrepancy explains the performance loss of our algorithm compared to the supervised approach and is bounded by the aforementioned $L_1$ distance.
To further bound this term, we define $e_i$ as 
the probability that $\hat{\pi}_i$ makes a confident (i.e., without querying the expert policy) but wrong action under $d_{\pi^*}$,
and introduce the following lemma:
\begin{lemma}\label{lemma:1}
    $||d_{\pi_i} - d_{\pi^*}||_1 \leq  2 T e_i$.
\end{lemma}
\begin{proof}
Let $\beta_{it}$ be the probability of querying the expert policy under $d_{\pi^*}^t$, $\Tilde{\epsilon}_{it}$ the error rate of $\hat{\pi}_i$ under $d_{\pi^*}^t$ (w.r.t. $\pi^*$), and $d$ any state distribution besides $d_{\pi^*}$. We can then express $d_{\pi_i}$ by:
\begin{align*}
d_{\pi_i} & = \prod_{t=1}^T (\beta_{it} + (1 - \beta_{it})(1 - \Tilde{\epsilon}_{it})) d_{\pi^*} \\
          & + (1 - \prod_{t=1}^T (\beta_{it} + (1 - \beta_{it})(1 - \Tilde{\epsilon}_{it})))d.
\end{align*}
The distance between $d_{\pi_i}$ and $d_{\pi^*}$ thus becomes
\begin{align*}
    & ||d_{\pi_i} - d_{\pi^*}||_1 \\
    = & (1 - \prod_{t=1}^T (\beta_{it} + (1 - \beta_{it})(1 - \Tilde{\epsilon}_{it}))) ||d-d_{\pi^*}||_1\\
    \leq & 2(1 - \prod_{t=1}^T (\beta_{it} + (1 - \beta_{it})(1 - \Tilde{\epsilon}_{it})))\\
    \leq & 2 \sum_{t=1}^T [1 - (\beta_{it} + (1 - \beta_{it})(1 - \Tilde{\epsilon}_{it}))]\\
    \leq & 2 \sum_{t=1}^T [\Tilde{\epsilon}_{it}(1-\beta_{it})]\\
    \leq & 2 \sum_{t=1}^T e_{it} \\
    = & 2 T e_i.
\end{align*}
The second inequality uses $1 - \prod_{t=1}^T x_t \leq \sum_{t=1}^T (1 - x_t)$, which holds when $x_t\in [0,1]$. 
\end{proof}
By applying Assumption~\ref{assumption:no-regret} and Lemma~\ref{lemma:1} to Eq.~\eqref{eqn:app:obj}, we derive the following inequality:
\begin{align}
J(\hat{\pi}) \leq T \big[ \gamma_N + \epsilon_N + \frac{2T \ell_{max}}{N} \sum_{i=1}^N e_i \big].\nonumber
\end{align}
Given a large enough $N$ ($N \in \Tilde{O}(T)$), by the no-regret assumption, we can further simplify the above as:
\begin{align}
J(\hat{\pi}) \leq T \big[ \epsilon_N + \frac{2T \ell_{max}}{N} \sum_{i=1}^N e_i \big] + O(1),\nonumber
\end{align}
which leads to our theorem:
\begin{theorem}\label{theorem:app:misp-l}
    For the proposed \neil algorithm, if $N$ is $\Tilde{O}(T)$, there exists a policy $\hat{\pi} \in \hat{\pi}_{1:N}$ s.t. $J(\hat{\pi}) \leq T \big[ \epsilon_N + \frac{2T \ell_{max}}{N} \sum_{i=1}^N e_i \big] + O(1)$.
\end{theorem}

By comparing Theorem~\ref{theorem:app:sup} and Theorem~\ref{theorem:app:misp-l}, it is obvious that the performance gap between \neil and the supervised approach is bounded by the term around $\frac{1}{N}\sum_{i=1}^N e_i$. We discuss its implications in Section~\ref{sec:theoretical_analysis} and show that, in practice, this performance gap can be controlled by carefully initializing the policy and choosing a more accurate confidence estimator.

\begin{remark}[Discussion about \mispls.]
In experiments, we consider a skyline instantiation of \neil, called \mispls. This instantiation is assumed with perfect confidence estimation and interaction design, such that it can precisely detect and correct its intermediate mistakes during parsing. Therefore, \mispls presents an upper bound performance (i.e., the \emph{tightest} cost bound) of \neil. This can be interpreted theoretically. In fact, for \mispls, $e_i$ is always zero since the system has ensured that its policy action is correct when it does not query the expert policy. In this case, $d_{\pi_i} = d_{\pi^*}$, so $\epsilon_N = \min_{\pi \in \Pi} \frac{1}{N} \sum_{i=1}^N \mathbb{E}_{s \sim d_{\pi^*}}[l(s, \pi)] = \min_{\pi \in \Pi} \mathbb{E}_{s \sim d_{\pi^*}}[l(s, \pi)]$. Therefore, according to Theorem~\ref{theorem:app:misp-l}, \mispls has a cost bound of:
\begin{equation*}
    J(\hat{\pi}) \leq T \epsilon_N + O(1),
\end{equation*}
where $\epsilon_N=\min_{\pi \in \Pi} \mathbb{E}_{s \sim d_{\pi^*}}[l(s, \pi)]$.

By comparing this bound with the cost bound in Theorem~\ref{theorem:app:sup}, it is observed that \mispls shares the same cost bound as the supervised approach (except for the inequality relation and the constant). This is explainable since \mispls indeed collects exactly the same training labels as the supervised approach.
\end{remark}

\subsection{Cost Bound of \neil Algorithm in Finite Sample Case}\label{app:subsec:finite}

The theorem in the previous section holds when the algorithm observes infinite trajectories in training. However, in practice, \neil observes the training loss from only a finite set of $m$ trajectories in each iteration. For this consideration, in the following discussion, we provide a proof of the cost bound of \neil under the finite sample case.

Denote $D_i$ as the $m$ trajectories collected in the $i$-th iteration and $\ell_i(\hat{\pi}_i) = \mathbb{E}_{s\sim D_i} [\ell(s, \hat{\pi}_i)]$. Applying the no-regret assumption (Assumption~\ref{assumption:no-regret}) allows us to bound the average expected policy training loss: $\frac{1}{N}\sum_{i=1}^N\mathbb{E}_{s\sim D_i}\big[\ell(s, \pi_i)\big] - \Tilde{\epsilon}_N \leq \Tilde{\gamma}_N$, where $\Tilde{\epsilon}_N=\min_{\pi\in \Pi}\frac{1}{N}\sum_{i=1}^N\mathbb{E}_{s\sim D_i}\big[\ell(s, \pi)\big]$ denotes the loss of the best policy in hindsight on the finite samples. 

Following Eq.~\eqref{eqn:app:mispl_derivation}, we need to switch the derivation from the expected loss of $\hat{\pi}_i$ over $d_{\pi_i}$ (i.e., $\mathbb{E}_{s \sim d_{\pi_i}}[\ell(s, \hat{\pi}_i)]$) to that over $D_i$ (i.e., $\mathbb{E}_{s \sim D_i}[\ell(s, \hat{\pi}_i)]$), the actual state distribution that $\hat{\pi}_i$ is trained on. To fill this gap, we introduce $Y_{ij}$ to denote the difference between the expected loss of $\hat{\pi}_i$ under $d_{\pi_i}$ and the average loss of $\hat{\pi}_i$ under the $j$-th sample trajectory with $\pi$ at iteration $i$. The random variables $Y_{ij}$ over all $i \in \{1,2,...,N\}$ and $j \in \{1,2,...,m\}$ are all zero mean, bounded in $[-\ell_{max},\ell_{max}]$ and form a martingale in the order of $Y_{11}, Y_{12}, ..., Y_{1m},Y_{21}, ..., Y_{Nm}$.
By Azuma-Hoeffding’s inequality \cite{azuma1967weighted,hoeffding1994probability}, $\frac{1}{mN}\sum_{i=1}^{N}\sum_{j=1}^{m} Y_{ij}\leq \ell_{max}\sqrt{\frac{2\log(1/\delta)}{mN}}$ with probability $1-\delta$. 
Following the derivations in Eq.~\eqref{eqn:app:mispl_derivation} and by introducing $Y_{ij}$, with probability of $1-\delta$, we obtain the following inequalities by definition:
\begin{align*}    
     & J(\hat{\pi})\nonumber\\
     \leq & \frac{T}{N} \sum_{i=1}^N \big[ \mathbb{E}_{s \sim d_{\pi_i}}[\ell(s, \hat{\pi}_i)] + \ell_{max} ||d_{\pi_i} - d_{\pi^*}||_1 \big] \\
    \leq & \frac{T}{N} \sum_{i=1}^N \Big[ \mathbb{E}_{s \sim D_i}[\ell(s, \hat{\pi}_i)] + \ell_{max} ||d_{\pi_i} - d_{\pi^*}||_1 \Big] \nonumber\\
    & + \frac{T}{mN}\sum_{i=1}^N\sum_{j=1}^m Y_{ij} \nonumber\\
    \leq & \frac{T}{N} \sum_{i=1}^N \Big[ \mathbb{E}_{s \sim D_i}[\ell(s, \hat{\pi}_i)] + \ell_{max} ||d_{\pi_i} - d_{\pi^*}||_1 \Big]   \nonumber\\
    & + \ell_{max}T\sqrt{\frac{2\log(1/\delta)}{mN}}\nonumber \\
     \leq & T \Big[\Tilde{\gamma}_N + \Tilde{\epsilon}_N + \ell_{max}\sqrt{\frac{2\log(1/\delta)}{mN}} \nonumber\\
     & + \frac{2\ell_{max}T}{N}\sum_{i=1}^{N}  e_i \Big].
\end{align*}

Notice that we need $mN$ to be at least $\Tilde{O}(T^2log(1/\delta))$, so that $\Tilde{\gamma_N}$ and $l_{max}\sqrt{\frac{2\log(1/\delta)}{mN}}$ are negligible.
This leads to the following theorem:
\begin{theorem}\label{theorem:app:misp-l-finite}
    For the proposed \neil algorithm, with probability at least $1-\delta$, when $mN$ is $\Tilde{O}(T^2log(1/\delta))$, there exists a policy $\hat{\pi} \in \hat{\pi}_{1:N}$ s.t. $J(\hat{\pi}) \leq T [ \Tilde{\epsilon}_N + \frac{2l_{max}T}{N}\sum_{i=1}^{N}   e_i ] + O(1)$.
\end{theorem}
The theorem shows that the cost of \neil can still be bounded in the finite sample setting. Comparing this bound with the bound under the infinite sample setting, we can observe that the bound is still related to $e_i$, the probability that $\hat{\pi}_i$ takes a confident but incorrect action under $d_{\pi^*}$.

\section{Implementation Details}\label{app:implementation_details}

\subsection{Interactive Semantic Parsing Framework}\label{app:misp}

Our system assumes an interactive semantic parsing framework to collect user feedback. In experiments, this is implemented by adapting MISP \cite{yao2019model}, an open-sourced framework\footnote{\url{https://github.com/sunlab-osu/MISP}.} that has demonstrated a strong ability to improve test-time parsing accuracy. In this framework, an agent is comprised of three components: a world model that wraps the base semantic parser and a feedback incorporation module to interpret user feeds and update the semantic parse, an error detector that decides whether to request for user intervention, and an actuator that delivers the agent's request by asking a natural language question, such that users without domain expertise can understand. 

We follow MISP's instantiation for text-to-SQL tasks to adopt a probability-based uncertainty estimator as the error detector, which triggers user interactions when the probability of the current decision is lower than a threshold. The actuator is instantiated by a grammar-based natural language generator. We use the latest version of MISP that allows multi-choice interactions to improve the system efficiency, i.e., when the parser's current decision is validated as wrong, the system presents multiple alternative options for user selection. An additional ``None of the above options'' option is included in case all top options from the system are wrong. Figure~\ref{fig:intro} shows an example of the user interaction. From there, the system can derive a correct decision to address its uncertainty (e.g., taking ``\texttt{Player}'' as a \texttt{WHERE} column).

{As a general interactive semantic parsing framework, MISP has its advantage of being generalizable to different kinds of semantic parsers (as long as their parsing process can be formulated as taking a sequence of actions in their respective action space) and various logical forms (e.g., lambda expressions). Although it could be non-trivial to instantiate such an interactive system, we note that it is a one-time effort for all datasets of the same logical form.}

\begin{remark}[Example of Non-sketch-based Parsers.] 
In addition to the example of the SQLova parser \cite{hwang2019comprehensive} that we provide in Section~\ref{sec:overview}, here we show how the EditSQL parser \cite{zhang2019editing} is formulated under MISP. 
Unlike SQLova, EditSQL does not assume any SQL sketch; it instead generates a SQL query ``token by token''.\footnote{EditSQL considers a column name as a single ``token'', although it may actually contain several words (e.g., \texttt{School/Club Team}).} Consider the SQL query in Figure~\ref{fig:intro}. EditSQL takes actions: $a_1$=``\texttt{SELECT}'', $a_2$=``\texttt{COUNT}'', $a_3$=``\texttt{(}'', $a_4$=``\texttt{School/Club Team}'', $a_5$=``\texttt{)}'', etc. Therefore, the action space of EditSQL consists of all SQL keywords, grammatical constituents (e.g., ``\texttt{(}''), and available table columns. In this case, MISP only validates semantically meaningful actions (including aggregators, operators, column names, etc.) while skipping others (including trivial symbols like ``\texttt{(}'' and most SQL keywords\footnote{Except \texttt{WHERE}, \texttt{GROUP BY}, \texttt{ORDER BY} and \texttt{HAVING}; see Appendix~\ref{app:subsec:editsql} for details.}).
\end{remark}

\begin{remark}[User Simulator.]
Our experiments train each system with simulated user feedback. To this end, we build a user simulator similar to the one used by \citet{yao2019model} in MISP, which can access the ground-truth SQL queries. It gives yes/no answer or selects a choice by directly comparing the sampled policy action with the true one in the gold query. When the true option is not presented within the system provided choices, the user is simulated to select ``None of the above options''. 
\end{remark}

\subsection{WikiSQL Experiment Details}\label{app:subsec:wikisql}

\begin{remark}[Dataset\&Model.]
Our main experiments consider the WikiSQL benchmark dataset \cite{zhong2017seq2sql},\footnote{\url{https://github.com/salesforce/WikiSQL}.} which contains 56,355/8,421/15,878 question-SQL query pairs in the training/validation/test set. We use exactly the same data split as \citet{zhong2017seq2sql}. 

We choose SQLova \cite{hwang2019comprehensive}, one of the open-sourced\footnote{\url{https://github.com/naver/sqlova}.} top-performing semantic parser on WikiSQL, as the base parser, which ensures reasonable model capability to study continual learning. Hyper-parameters are set the same as the ones recommended by the SQLova authors on their GitHub repository,\footnote{\url{https://github.com/naver/sqlova\#running-code}.} except that we use a learning rate of 1e-5 for fine-tuning the BERT model. Empirically we found out this relatively larger learning rate can greatly accelerate the model learning without affecting the model performance significantly. The total number of model parameters is around 118M, with 110M from BERT-Base (Uncased)\footnote{\url{https://github.com/google-research/bert\#pre-trained-models}.} and 8M from the SQLova parser side.

Early stop is used to accelerate model training in each training iteration. Specifically, we stop model training if it does not show improvement on the validation set for a consecutive number of epochs. We set this number to 10 before the 30-th training iteration when the total training data is in a relatively small size, and decay it to 5 after the 30-th iteration.
We follow SQLova when preprocessing the WikiSQL data.
\end{remark}

\begin{remark}[Experimental Setup.]
We study a ``continual learning'' problem and experiment various systems with three initialization settings, as suggested by our theoretical analysis (Section~\ref{sec:theoretical_analysis}). Specifically, we use 10\% (5,636 pairs), 5\% (2,818 pairs), and 1\% (564 pairs) of the total training data for parser initialization, respectively. 

In each initialization setting, the remaining training data is used to simulate user questions that a system receives after deployment. The user questions come in a random order. We repeat three random runs (i.e., three random orders of user questions) and report the average system performance. Notice that, we ensure each system receive the same user question (but may have different user feedback depending on their interaction designs) during iterative training, for a fair comparison. Systems update (retrain) their base semantic parsers periodically for every 1,000 user questions. 
\end{remark}

\begin{remark}[Metrics.]
In the end of each iteration, we evaluate the system's performance, including:
\begin{itemize}[itemsep=0.5pt,topsep=1pt]
    \item Parsing accuracy. We measure the query match accuracy (i.e., logical form accuracy) using the script from SQLova implementation.
    \item An accumulated number of user/expert annotations (introduced in Section~\ref{subsec:exp_results}). Different systems request different kinds of user/expert annotations. Therefore, even when serving the user on the same user question, different systems require different numbers of annotations. This metric sums up the total number of annotations that each system has requested after each training iteration.
\end{itemize}
Calculating the aforementioned metrics allow us to plot Figure~\ref{fig:exp} and Figure~\ref{fig:exp_dev}.
\end{remark}

\begin{remark}[Compute.]
We complete experiments on Nvidia GeForce RTX 2080Ti (11GB). Models are all implemented using PyTorch.\footnote{\url{https://pytorch.org/}.} The run time for each training iteration varies depending on the accumulated training data size. To finish the 50+ iterations of (re-)training, each system takes around 15 days. In the weak 1\% initialization case, the Binary User baseline takes less time (around 10 days), since most of its predicted queries are wrong and thus are not included into its training data. 
\end{remark}

\subsection{EditSQL Experiment Details}\label{app:subsec:editsql}

\begin{remark}[Data\&Model.]
The Spider dataset \cite{yu2018spider} contains 8,421 question-SQL pairs for training and 1,034 pairs for validation.\footnote{\url{https://github.com/taoyds/spider}.} The test set is not publicly available and is thus not used in our experiments.

We choose EditSQL \cite{zhang2019editing} as the base semantic parser,\footnote{\url{https://github.com/ryanzhumich/editsql}.} since it is one of the open-sourced state-of-the-art models on Spider. All hyper-parameters are set following \cite{zhang2019editing}. Pre-trained BERT model is also used. Totally there are around 120M parameters in the model, with 110M from the BERT-Base (Uncased)\footnote{\url{https://github.com/google-research/bert\#pre-trained-models}.} and 10M from the EditSQL parser side.
Early stop is additionally used to accelerate model training. Specifically, we stop model training when it does not show improvement on validation for 5 consecutive epochs.

In the data preprocessing step, EditSQL transforms each gold SQL query into a sequence of tokens, where the \texttt{From} clause is removed and each column \texttt{Col} is prepended by its paired table name, i.e., \texttt{Tab.Col}. However, we observe that sometimes this transformation is not convertible. For example, consider the question ``what are the first name and last name of all candidates?'' and its gold SQL query:
``\texttt{SELECT T2.first\_name ,  T2.last\_name FROM candidates AS T1 JOIN people AS T2 ON T1.candidate\_id = T2.person\_id}''.
EditSQL transforms this query into :
``\texttt{select people.first\_name , people.last\_name}''. 
The transformed sequence accidentally removes the information about table \texttt{candidates} in the original SQL query, leading to semantic meaning inconsistent with the question.
When using such erroneous sequences as the gold targets in model training, we cannot simulate consistent user feedback, e.g., when the user is asked whether her query is relevant to the table \texttt{candidates}, the simulated user cannot give an affirmative answer based on the transformed sequence. To avoid inconsistent user feedback, we remove question-SQL pairs whose transformed sequence is inconsistent with the original gold SQL query, from the training data. This can be easily done by using EditSQL's post-processing script to convert a preprocessed sequence back to the SQL format. Only when the converted query is the same as the original one, the transformation is consistent. This reduces the size of the training set from 8,421 to 7,377. The validation set is kept untouched for fair evaluation.

The implementation of interactive semantic parsing for EditSQL is the same as Section~\ref{app:misp}, except that, in order to cope with the complicated structure of Spider SQL queries, for columns in \texttt{WHERE}, \texttt{GROUP BY}, \texttt{ORDER BY} and \texttt{HAVING} clauses, we additionally provide an option for the user to ``remove'' the clause, e.g., removing a \texttt{WHERE} clause by picking the ``The system does not need to consider any conditions.'' option. We also adjust the ``semantic unit'' definition in MISP\footnote{\url{https://github.com/sunlab-osu/MISP/blob/multichoice_q/MISP_SQL/tag_seq_logic.md}.} to deal with the autoregressive decoding of EditSQL. For example, instead of asking first about a \texttt{SELECT} column and then about its aggregator, we define one semantic unit to inquire about both the column and its aggregator.

To instantiate \neil, the confidence threshold $\mu$ is 0.995 as we observe that EditSQL tends to be overconfident.
\end{remark}

\begin{remark}[Experimental Setup.]
We experiment with one initialization setting, using 10\% of the total training data (i.e., 737 question-SQL pairs), and systems update (retrain) their base semantic parsers periodically for every 1,000 user questions as in WikiSQL experiments. We report system performance averaged over three random runs (i.e., three random orders of user questions).

We also tried using more training data for initialization. However, since the total training data in Spider is very limited in size, more initialization data means fewer data for simulating online user questions and conducting continual learning. This leads to less clear experimental observations (e.g., even the Full Expert system shows fluctuation, probably due to data redundancy or an issue with model architecture capability). Therefore, we only focus on the 10\% initialization setting.
\end{remark}

\begin{remark}[Metrics.]
We measure each system similarly as in WikiSQL experiments. For parsing performance, we calculate the exact match accuracy using scripts from the EditSQL implementation.
\end{remark}

\begin{remark}[Compute.]
We complete experiments on Nvidia GeForce RTX 2080Ti (11GB). Models are implemented using PyTorch. The run time for each training iteration varies depending on the accumulated training data size. Finishing the whole iterative learning takes around 5 days for all systems.
\end{remark}

\section{Additional Experimental Results}\label{app:detailed_exp}

\subsection{Additional SQLova Results}
Figure~\ref{fig:exp_dev} shows different systems' performance on WikiSQL validation set. 
For Binary User(+Expert), it is hard to quantify ``one annotation'', which varies according to the actual database size and the query difficulty. As a compromise, we approximate this number by calculating it in the same way as Full Expert, with the assumption that in general validating execution results is as hard as validating the SQL query itself.

We also show in Figure~\ref{fig:q_curve} the average number of annotations (i.e., user interactions) that \mispl requires per question during the iterative training. Overall, as the base parser is further trained, our system tends to request fewer user interactions. In most cases throughout the training, the system requests no more than one user interaction, demonstrating the annotation efficiency of our \neil algorithm.

\begin{figure*}[ht!]
    \centering
    \includegraphics[width=0.95\linewidth]{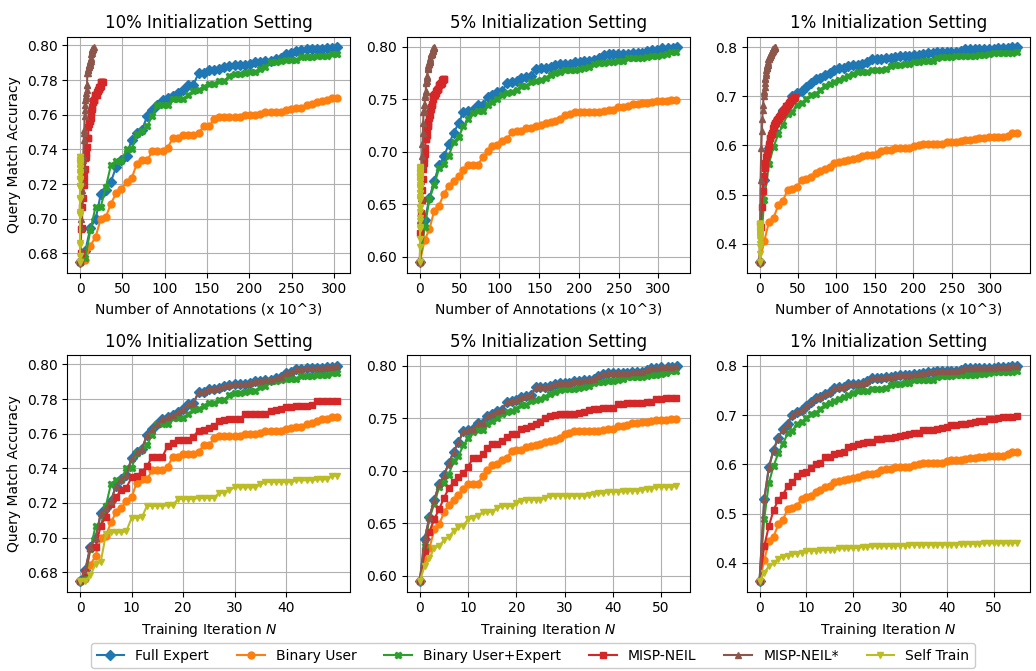}
    \caption{Parsing accuracy on WikiSQL validation set when systems are trained with various numbers of user/expert annotations (top) and for different iterations (bottom). We experiment systems with three initialization settings, using 10\%, 5\% and 1\% of the training data respectively.}
    \label{fig:exp_dev}
\end{figure*}

\begin{figure*}[ht]
    \centering
    \begin{subfigure}{0.7\columnwidth}
        \centering
        \includegraphics[width=\linewidth]{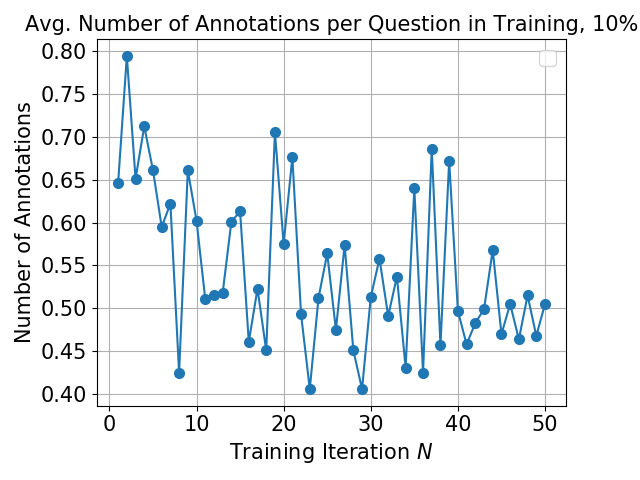}
    \end{subfigure}%
    \begin{subfigure}{0.7\columnwidth}
        \centering
        \includegraphics[width=\linewidth]{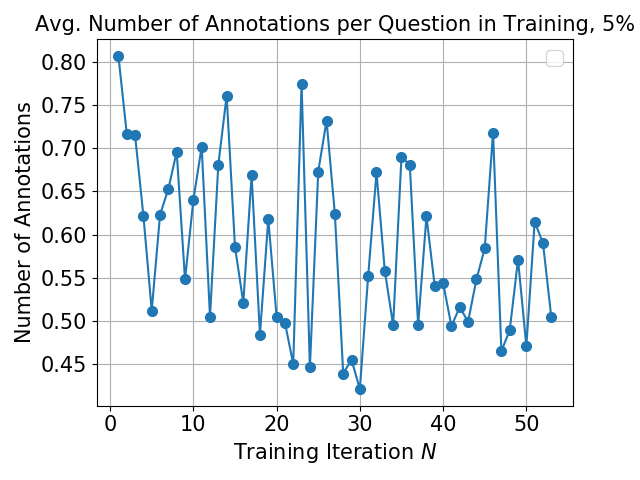}
    \end{subfigure}%
    \begin{subfigure}{0.7\columnwidth}
        \centering
        \includegraphics[width=\linewidth]{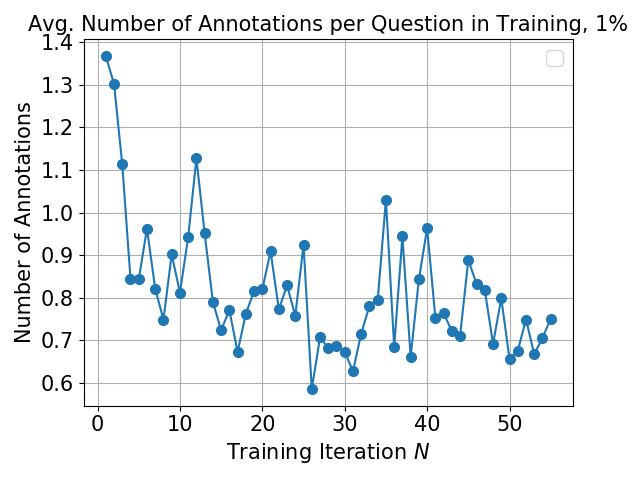}
    \end{subfigure}
    \caption{Average number of user annotations/interactions that \mispl requests for each user question during iterative training (on WikiSQL), when the parser is initialized using 10\%, 5\% and 1\% of training data.}
    \label{fig:q_curve}
\end{figure*}

\subsection{Connection to Theoretical Analysis}\label{app:subsec:theoretical_exp}
As we proved in Section~\ref{sec:theoretical_analysis}, the performance gap between our proposed \neil algorithm and the supervised approach is mainly decided by $\frac{1}{N}\sum_{i=1}^N e_i$, an average probability that $\hat{\pi}_i$ makes a confident but wrong decision under $d_{\pi^*}$ (i.e., given a gold partial parse) over $N$ training iterations. 
More specifically, from our proof of Lemma~\ref{lemma:1}, $e_i$ can be expressed as: 
$$
e_i = \frac{1}{T}\sum_{t=1}^T e_{it} = \frac{1}{T}\sum_{t=1}^T \Tilde{\epsilon}_{it}(1 - \beta_{it}),
$$
where $\Tilde{\epsilon}_{it}$ denotes policy $\hat{\pi}_i$'s conditional error rate under $d_{\pi^*}^t$ when it does not query the expert (i.e., being confident about its own action) at step $t$, and $1-\beta_{it}$ denotes the probability that $\hat{\pi}_i$ does not query the expert under $d_{\pi^*}^t$. $\Tilde{\epsilon}_{it}(1 - \beta_{it})$ thus represents a joint probability that $\hat{\pi}_i$ makes confident but wrong action under $d_{\pi^*}^t$ at step $t$.

To show a reflection of our theoretical analysis on the experiments, we present the values of the following three variables during training: (1) $\Tilde{\epsilon}_i = \frac{1}{T}\sum_{t=1}^T \Tilde{\epsilon}_{it}$, the average value of $\Tilde{\epsilon}_{it}$ over $T$ time steps. A smaller $\Tilde{\epsilon}_i$ implies a lower conditional error rate and thus a smaller $e_i$ and a smaller performance gap. (2) $\beta_i = \frac{1}{T}\sum_{t=1}^T \beta_{it}$, the average value of $\beta_{it}$ over $T$ time steps. A smaller $\beta_i$ (i.e., a larger $1 - \beta_i$) means a smaller probability that $\hat{\pi}_i$ queries the expert (i.e., being more confident). This could lead to a larger $e_i$ and thus a larger performance gap. (3) $e_i$ as defined previously. A smaller $e_i$ indicates a smaller performance gap between our algorithm and the supervised approach.

\begin{figure*}[ht]
    \centering
    \begin{subfigure}{0.7\columnwidth}
        \includegraphics[width=\linewidth]{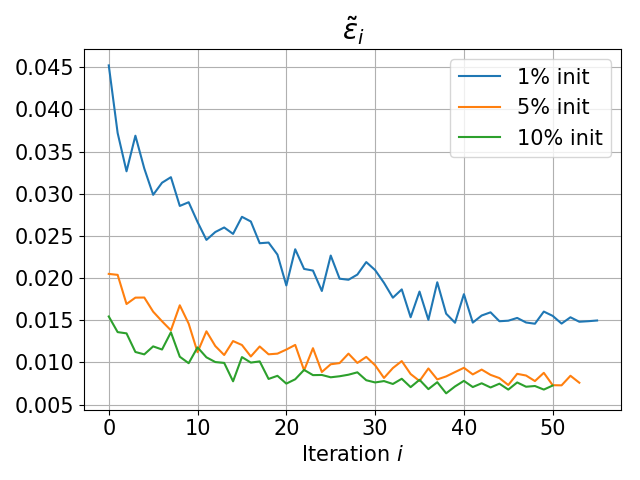}
        \caption{}
        \label{subfig:epsilon_i}
    \end{subfigure}%
    \begin{subfigure}{0.7\columnwidth}
        \includegraphics[width=\linewidth]{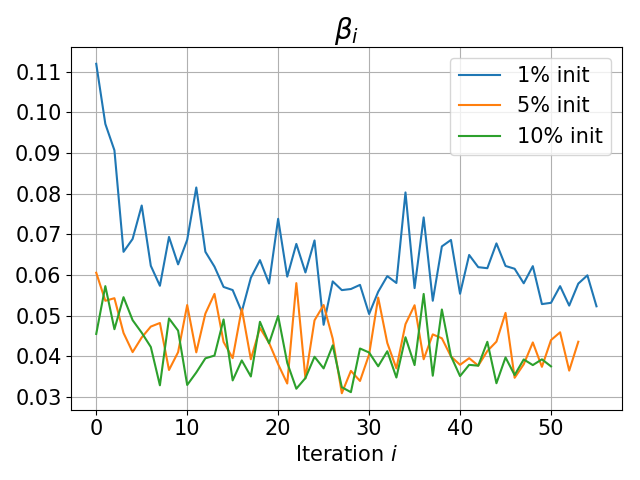}
        \caption{}
        \label{subfig:beta_i}
    \end{subfigure}%
    \begin{subfigure}{0.7\columnwidth}
        \includegraphics[width=\linewidth]{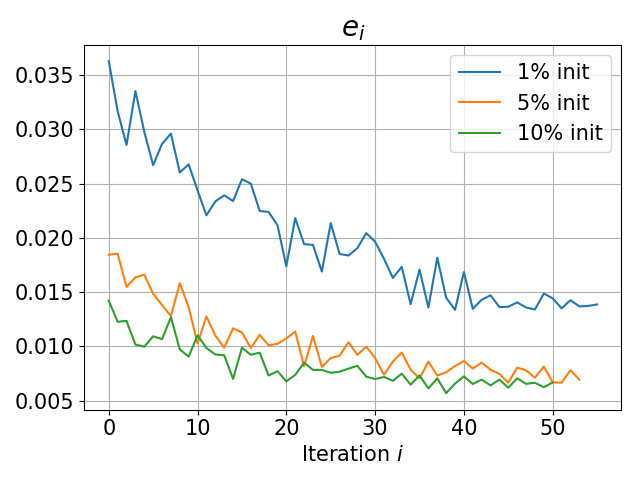}
        \caption{}
        \label{subfig:e_i}
    \end{subfigure}
    \caption{The values of $\Tilde{\epsilon}_i$ (a), $\beta_i$ (b) and $e_i$ (c) in \mispl throughout the training (on WikiSQL validation set), under different initialization settings.}
    \label{fig:theoretical_analysis}
\end{figure*}

We plot the results of our \mispl system (based on SQLova) in Figure~\ref{fig:theoretical_analysis}. For all initialization settings, we observe that the base parser tends to make more confident actions under a gold partial parse (i.e., decreasing $\beta_i$) when it is trained for more iterations. Meanwhile, the error rate of its confident actions under a gold partial parse is also reduced (i.e., decreasing $\Tilde{\epsilon}_i$). When combining the two factors, $e_i$ is shown to keep decreasing, implying that with more iterations that the parser is trained, it gets a tighter cost bound and better performance.

Finally, we notice that a differently initialized parser can end up with different performance. This is reasonable since a better initialized parser presumably should have a better overall error rate. This is also consistent with our observation in the main experimental results (Section~\ref{subsec:exp_results}).

\end{document}